\documentclass[letterpaper, 10 pt, conference]{ieeeconf}
\IEEEoverridecommandlockouts  % DO NOT CHANGE THIS
\usepackage{array}
\usepackage{booktabs}
\usepackage{amsfonts}
\usepackage{algorithm}
\usepackage[noend]{algpseudocode}
\usepackage{stmaryrd}
\usepackage{amssymb}
\usepackage{amsmath}
\usepackage{mathrsfs} 
\usepackage{graphicx}
\usepackage{cite}
\usepackage{paralist}

\newtheorem{lemma}{Lemma}
\newtheorem{theorem}{Theorem}
\newtheorem{proposition}{Proposition}
\newtheorem{definition}{Definition}

\newcommand{\setcond}[2]{\{\, #1 \mid #2 \,\}}

\DeclareMathOperator{\proj}{proj}
\newcommand{\reals}{\mathbb{R}}
\newcommand{\posreals}{\reals_{\geq 0}}
\newcommand{\naturals}{\mathbb{N}}

\newcommand{\cost}{\mathsf{Cost}}

\newcommand{\pruned}{\mathord{\gtrdot}}
\newcommand{\unrolled}{\otimes}

\newcommand{\proc}[1]{\textsc{#1}}

\newcommand{\dfa}{\mathcal{A}}
\newcommand{\mdpst}[1][M]{\mathcal{#1}}
\newcommand{\rcap}{\mathsf{cap}}
\newcommand{\acc}{F}
\newcommand{\lang}{\mathcal{L}}
\newcommand{\distr}{\text{Distr}}
\newcommand{\nature}{\gamma}
\newcommand{\natures}{\Gamma}
\newcommand{\strategy}{\sigma}
\newcommand{\strategies}{\Sigma}
\newcommand{\fpath}{\xi}
\newcommand{\fpaths}{\text{FPaths}}
\newcommand{\ipaths}{\text{Paths}}
\newcommand{\last}{\mathit{lst}}

\newcommand{\cylinder}[1]{\mathit{Cyl}(#1)}
\newcommand{\prob}{\mathrm{Pr}}

\newcommand{\AP}{\textsf{AP}} 
\newcommand{\ltl}{\text{LTL}}
\newcommand{\ltlf}{\ltl_{f}}
\newcommand{\llast}{\text{last}}
\newcommand{\lnext}{\bigcirc}
\newcommand{\luntil}{\mathbin{\mathcal{U}}}
\newcommand{\lfinally}{\Diamond}
\newcommand{\lglobally}{\Box}

\iftrue
    \usepackage{todonotes}
    \newcommand{\at}[1]{\todo[inline,color=teal!10,caption={AT}]{\textbf{AT:} #1}}
    \newcommand{\yh}[1]{\todo[inline,color=orange!10,caption={YH}]{\textbf{YH:} #1}}
    \newcommand{\ly}[1]{\todo[inline,color=blue!10,caption={LY}]{\textbf{LY:} #1}}
    \newcommand{\py}[1]{\todo[inline,color=yellow!10,caption={PY}]{\textbf{PY:} #1}}
    \newcommand{\czm}[1]{\todo[inline,color=red!10,caption={CZM}]{\textbf{CZM:} #1}}

\else
    \newcommand{\at}[1]{}
    \newcommand{\yh}[1]{}
    \newcommand{\ly}[1]{}
    \newcommand{\py}[1]{}
    \newcommand{\czm}[1]{}
\fi

\title{Resource-Constrained Robotic Planning in the face of Mixed Uncertainty}

\author{Yihao Yin$^{1,2,3}$, Pian Yu$^{4}$, Andrea Turrini$^{2}$, Zhiming Chi$^{2,3}$, Yong Li$^{2}$, and Lijun Zhang$^{2,3}$
\thanks{*Corresponding author: {pian.yu@ucl.ac.uk, liyong@ios.ac.cn}}%
\thanks{$^{1}$Hangzhou Institute for Advanced Study (HIAS), UCAS, China.}%
\thanks{$^{2}$Key Laboratory of System Software (Chinese Academy of Sciences), Institute of Software Chinese Academy of Sciences, Beijing, China.}%
\thanks{$^{3}$University of Chinese Academy of Sciences, Beijing, China.}%
\thanks{$^{4}$University College London, London, UK.}%
}

\begin{document}

\maketitle
\thispagestyle{empty}
\pagestyle{empty}
\begin{abstract}
    Robots operate under significant uncertainty, from quantifiable noise to unquantifiable unknowns, and must account for strict operational constraints, such as limited resources. 
    In this paper, we consider the problem of synthesizing robust strategies to guide a robot's actions in fulfilling a given task, while ensuring the system never exhausts its resources.
    To solve this problem, we first model the robotic system as a Consumption Markov Decision Process with Set-valued Transitions (CMDPST), a unified framework modelling nondeterministic actions, quantifiable and unquantifiable uncertainty, and resource consumption.
    Then, we combine the CMDPST with the task specification, expressed as a Linear Temporal Logic over finite traces ($\ltlf$) formula.
    Lastly, we address the resource-constrained optimal robust strategy synthesis problem, which aims to synthesize a strategy that maximizes the probability of satisfying the $\ltlf$ objective without resource exhaustion. 
    
    Our solution involves two techniques: a direct unrolling-based method and a more efficient, optimized approach that leverages state-space pruning for better performance.   
    Experiments on a warehouse transportation network show the effectiveness of the proposed solutions.
\end{abstract}
\section{Introduction}
\label{sec:introduction}
Autonomous systems are increasingly adopted in safety-critical domains, from autonomous driving\cite{levinson2011towards,xu2014motion} and planetary exploration\cite{arm2023scientific,bahraini2021autonomous} to medical robotics \cite{dupont2021decade}. 
A core challenge in these systems is decision-making under uncertainty while ensuring strict adherence to safety and
operational constraints, such as limited energy, memory, or computational resources. 
These systems are often modeled as stochastic discrete-time dynamical systems tasked with achieving complex goals, and must guarantee not only the satisfaction of such goals but also the continuous maintenance of critical safety conditions during execution\cite{ABATE20082724}.

Markov Decision Processes (MDPs) have emerged as a standard formalism for modeling uncertain systems, thanks to their ability to capture probabilistic outcomes and sequential decision-making~\cite{Puterman2014,MausamKolobov2012}. 
However, traditional MDPs are deficient in directly encoding hard resource constraints or environmental uncertainties~\cite{Cognetti2018,Thrun2005,DuToit2011}. 
Recent extensions such as resource-constrained MDPs~\cite{Blahoudek2020,blahoudek2022efficient} have sought to address these limitations by incorporating cost bounds or modeling adversarial uncertainties. 
Yet, most of these approaches focus on expected cost optimization and assume known probabilistic models, which limits their applicability in settings where the environment is uncertain or adversarial and safety guarantees are paramount.

Concurrently, formal languages like Linear Temporal Logic ($\ltl$)~\cite{Pnueli1977} and its finite-trace variant $\ltlf$~\cite{degiacomo2013linear} have proven to be powerful tools for specifying high-level mission objectives in autonomous systems. 
These specifications allow expressing rich temporal requirements such as ``eventually deliver the package to location B while avoiding obstacle zones'' or ``do not deplete fuel before reaching the charging station''.
By combining MDPs with $\ltl$ and $\ltlf$, recent work has enabled efficient task planning in stochastic systems with temporal objectives~\cite{Ding2014,Schillinger2019,Guo2018,cai2021reinforcement,wongpiromsarn2023formal}. 
However, these methods often do not account for resource constraints.
In addition, in many real-world robotic scenarios, uncertainty stems from both quantifiable sources, like sensor noise, and unquantifiable ones, such as adversarial environments or human error ~\cite{muvvala2024stochastic}.
These aspects cannot be adequately captured by ordinary MDPs, which assume precisely specified probabilistic transitions.

To address these challenges, we propose in this work \emph{Consumption Markov Decision Processes with Set-valued Transitions} (CMDPSTs), a
unified framework modelling nondeterministic actions, quantifiable and unquantifiable uncertainty, and resource consumption. This model is  built upon \emph{Markov Decision Processes with Set-valued Transitions} (MDPSTs)~\cite{trevizan2007,trevizan2008, DBLP:conf/ijcai/Yu0GKV24,unquantified2025}. 
MDPSTs generalize MDPs by allowing transitions to be defined as sets of possible successor distributions, thereby capturing both quantifiable (e.g., stochastic) and unquantifiable (e.g., adversarial or nondeterministic) uncertainties. In particular, probabilistic weights in MDPST capture quantifiable uncertainty, while set-valued transitions model unquantifiable uncertainty without assuming precise probability distributions.
Recent work has demonstrated the effectiveness of MDPSTs in robust planning with temporal-extended objectives \cite{unquantified2025}.  
To further capture the notion of limited and renewable resources, we extend the MDPST formalism by introducing ``resource levels", ``reload states", and ``capacity bounds", allowing one to model systems that must periodically replenish critical resources to safely complete their tasks.

In this work, we propose to formulate the resource-constrained robust planning problem for robotic systems in the face of mixed (quantifiable and unquantifiable) uncertainty as the robust strategy synthesis problem for CMDPSTs. In addition, $\ltlf$ is employed as a concise and compact specification language.

The main contributions of this paper are summarized as follows. 
i) We introduce CMDPSTs (a
unified modelling framework) and formulate a novel \emph{resource constrained optimal
robust strategy synthesis} problem, which aims to compute a strategy that optimizes $\ltlf$ satisfaction while preventing resource exhaustion. 
ii) A na\"ive approach is first proposed to solve the problem by converting the CMDPST into an unrolled MDPST and applying an existing algorithm. To improve efficiency, we then introduce an optimization that prunes the CMDPST prior to unrolling, significantly reducing the state space required for synthesis. Experimental results show that our pruning method consistently achieves notable runtime improvements over the na\"ive approach, all while preserving the correctness of the synthesized strategy. 

The remainder of the paper is organized as follows.  
Section~\ref{sec:motivatingExample} presents a motivating example illustrating the challenges of resource-constrained planning under mixed uncertainty.  
Section~\ref{sec:preliminaries} reviews preliminaries on $\ltlf$ and MDPSTs.  
Section~\ref{sec:cmdpsts} formalizes Consumption MDPSTs.  
Section~\ref{sec:unroll} describes the na\"ive unrolling-based synthesis pipeline. 
Section~\ref{sec:optimizedStrategySynthesis} develops the optimized approach.
Section~\ref{sec:experiments} reports experimental evaluations on the warehouse benchmarks.  
Section~\ref{sec:conclusion} concludes the paper. 

\section{Motivating Example}
\label{sec:motivatingExample}

\begin{figure}[t]
\centering

\includegraphics[width=0.7\linewidth]{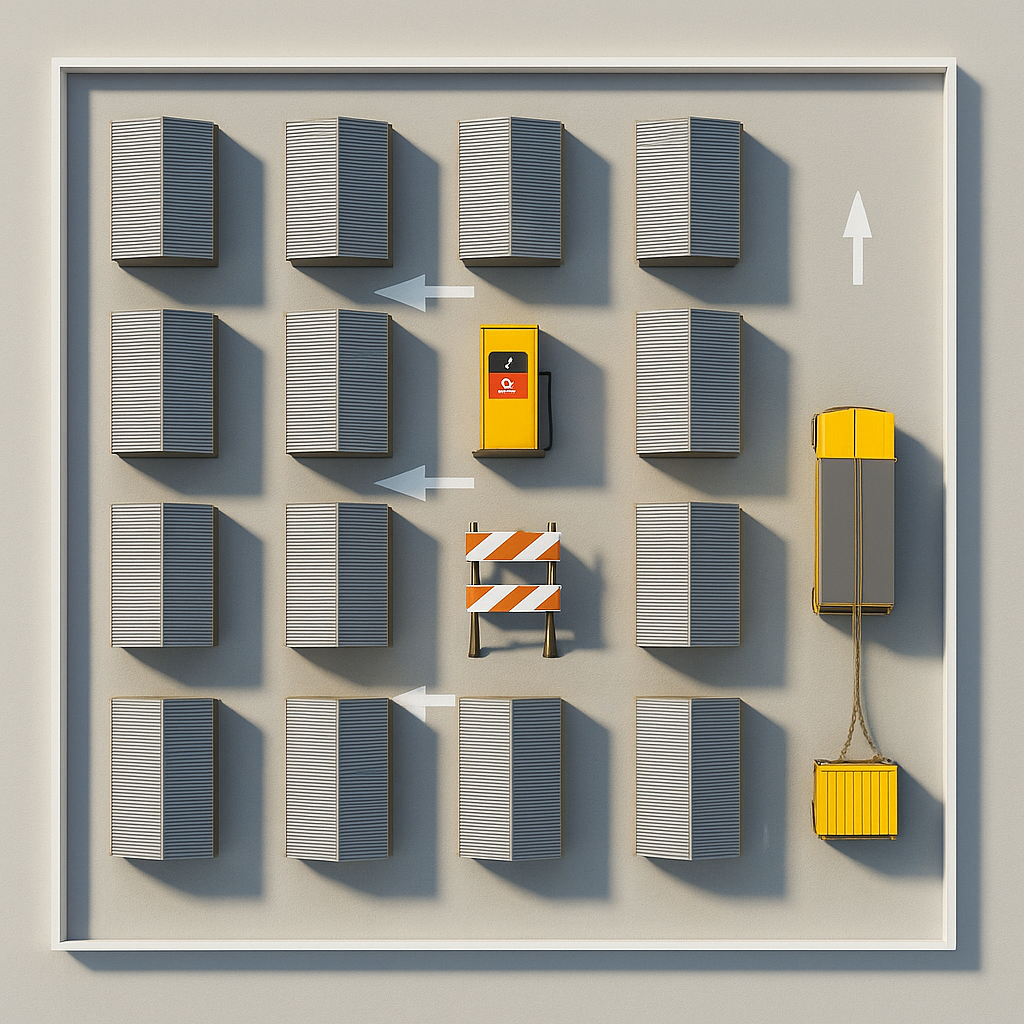}
% \vspace{-5pt} 
\caption{Warehouse network}
% \vspace{-20pt} 
\label{fig:Warehouse}
\end{figure}
As an example of a Consumption MDPST, we consider a warehouse transportation scenario where an automated guided vehicle (AGV) is required to transport goods across multiple warehouses; 
see Fig.~\ref{fig:Warehouse} for an example. 
Within this warehouse network, we consider three typical situations:
\begin{inparaenum}
\item
    some warehouses are temporarily inaccessible due to construction or maintenance (obstacle states);
\item
    some warehouses are equipped with charging stations (reload states), where the AGV can replenish its battery; and
\item
    certain warehouses are designated as task-related locations, such as pickup, transfer, and delivery points.
\end{inparaenum}
For simplicity, we assume that each warehouse is of exactly one of these types.

The AGV is required to accomplish a transportation task with temporal constraints such as reach the warehouse $W_{p}$ to pick up goods, move to warehouse $W_{t}$ for transfer, or deliver the goods to the target warehouse $W_{d}$, while always avoiding obstacle warehouses $W_{o}$.

Each movement of the AGV (forward, backward, left, or right) consumes a certain amount of energy. 
Since the distance between warehouses may be long, the AGV must carefully plan its route to avoid energy depletion, which may require visiting reload warehouses (charging stations). 
Meanwhile, external disturbances may occur: 
for instance, a human operator may intervene and manually redirect the AGV, causing it to deviate from the planned route. 
Thus, the AGV’s motion is subject to both stochastic uncertainty (e.g., execution noise) and nondeterministic disturbances (e.g., human intervention).

Unlike simple shortest-path planning, this scenario requires the AGV to satisfy a \emph{temporal-extended task}, which can be naturally formalized by formal languages like $\ltlf$. 
The strategy controlling the AGV should not merely identify a shortest route to the final destination, but it must also respect resource constraints. 
Our goal is to compute robust strategies that guarantee the AGV never runs out of energy, while maximizing the probability of successfully completing the temporal task despite environmental or human-induced disturbances.

\section{Preliminaries}
\label{sec:preliminaries}

In this section we first briefly recall  Linear Temporal Logic on Finite Traces ($\ltlf$)~\cite{baier2006planning,degiacomo2013linear} and its equivalent deterministic finite automata representation, and then the Markov Decision Processes with Set-valued Transitions (MDPSTs) modelling framework~\cite{trevizan2007,trevizan2008}.

% \at{To save space, we can use inline lists instead of itemized lists for the definitions of DFAs and MDPSTs}

\subsection{$\ltlf$}

$\ltlf$~\cite{degiacomo2013linear} has the same syntax as LTL~\cite{Pnueli1977}, but is interpreted on \emph{finite} yet unbounded horizons.
%It is suitable for modeling terminating systems such as robotic task sequences and business process executions.
The syntax of an $\ltlf$ formula over a finite set of atomic propositions $\AP$ is defined as
% \vspace{-3pt} 
\[
    \varphi ::= p \in \AP\mid \neg \varphi \mid \varphi \land \varphi \mid \lnext \varphi \mid \varphi \luntil \varphi.
\]
%where $p \in \AP$ represents an atomic proposition that captures basic system states or properties.
Here $\lnext$ (strong Next) and $\luntil$ (Until) are temporal operators.
We also use common operators, such as $\lfinally$ (eventually) and $\lglobally$ (always), where $\lfinally \varphi \equiv \text{true} \luntil \varphi$ and $\lglobally \varphi \equiv \neg \lfinally \neg \varphi$.

Let $\pi = \pi_{0} \pi_{1} \pi_{2} \dots \pi_{N} \in (2^\AP)^*$ be a finite trace and $\varphi$ be an $\ltlf$ formula.
We denote by $\pi[i\ldots] = \pi_i \ldots \pi_N$ the fragment that starts at position $i$.
The semantics of $\varphi$ is defined inductively as follows:
% \vspace{-5pt} 
\begin{align*}
    \pi \models p & \iff p \in \pi_{0} ,\\ 
 \pi \models \neg \varphi & \iff \pi \not \models \varphi, \\
 \pi \models \varphi_{1} \land \varphi_{2} & \iff \text{$\pi \models \varphi_{1}$ and $\pi\models \varphi_{2}$}, \\
 \pi \models \lnext \varphi & \iff \text{$N >0 $ and $\pi[1\ldots] \models \varphi$}, \\
 \pi \models \varphi_{1} \luntil \varphi_{2} & \iff \text{$\exists 0 \leq k \leq N. \pi[k\ldots]\models \varphi_{2}$ and}, \\
    & \hphantom{{}\iff{}} \quad \forall 0 \leq j < k. \pi[j\ldots] \models \varphi_{1}.
%\pi, i \models \llast & \iff i = N
\end{align*}

The language of $\varphi$, denoted $[\varphi]$, is
the set of all finite traces over $2^{\AP}$ that satisfy $\varphi$, i.e., $[\varphi] = \setcond{\pi \in (2^{\AP})^{*}}{\pi \models \varphi}$.
\iffalse
The semantics of an $\ltlf$ formula $\varphi$ is defined over finite but unbounded traces $\pi = \pi_{0} \pi_{1} \pi_{2} \dots \pi_{N}$, where each $\pi_{i} \subseteq \AP$ denotes the set of atomic propositions holding at position $i$.
The satisfaction relation $\pi, i \models \varphi$ is defined recursively as:
\begin{align*}
\pi, i \models p & \iff p \in \pi_{i} \\
\pi, i \models \neg \varphi & \iff \pi, i \not \models \varphi \\
\pi, i \models \varphi_{1} \land \varphi_{2} & \iff \text{$\pi, i \models \varphi_{1}$ and $\pi, i \models \varphi_{2}$} \\
\pi, i \models \lnext \varphi & \iff \text{$i < N$ and $\pi, i+1 \models \varphi$} \\
\pi, i \models \varphi_{1} \luntil \varphi_{2} & \iff \text{$\exists i \leq k \leq N. \pi, k \models \varphi_{2}$ and} \\
    & \hphantom{{}\iff{}} \quad \forall i \leq j < k. \pi, j \models \varphi_{1} \\
\pi, i \models \llast & \iff i = N
\end{align*}
We say that $\pi$ satisfies $\varphi$, denoted $\pi \models \varphi$, if $\pi, 0 \models \varphi$.
We denote by $[\varphi]$ the set of all finite traces that satisfy $\varphi$, i.e., $[\varphi] = \setcond{ \pi \in (2^\AP)^{*}}{\pi \models \varphi}$.
\fi

For an $\ltlf$ formula $\varphi$ over $\AP$, one can construct a deterministic finite automaton (DFA) $\dfa_{\varphi} = (\Sigma, Q, \bar{q}, \delta, \acc)$ recognizing $[\varphi]$, where $\Sigma = 2^{\AP}$ is the alphabet, $Q$ is a finite set of states, $\bar{q} \in Q$ is the initial state, $\delta \colon Q \times \Sigma \to Q$ is the transition function, and $\acc \subseteq Q$ is the set of accepting states.
A run $\rho = q_{0} q_{1} \dots q_{N + 1} \in Q^{+}$ over a word $w = w_{0} w_{1} \dots w_{N} \in \Sigma^{*}$ is a sequence of states such that $q_{0} = \bar{q}$ and $q_{i+1} = \delta(q_{i}, w_{i})$ for each $0 \leq i \leq N$.
The word $w = w_{0} w_{1} \dots w_{N} \in \Sigma^{*}$ is accepted by $\dfa_{\varphi}$ if $q_{N+1} \in \acc$.

\subsection{MDPSTs}

We now introduce MDPSTs, a framework capable of modeling both quantifiable and unquantifiable uncertainties in robotics~\cite{DBLP:conf/ijcai/Yu0GKV24,unquantified2025}.
An MDPST is a tuple $\mdpst = (S, \bar{s}, A, \mathcal{F}, \mathcal{T}, L)$
where $S$ is a finite set of states, $\bar{s} \in S$ is the designated initial state, $A$ is a finite set of actions ($A(s) \subseteq A$ denotes applicable actions at $s$), $\mathcal{F} \colon S \times A \to 2^{2^{S}}$ is the set-valued transition function, $\mathcal{T} \colon S \times A \times 2^{S} \to (0, 1]$ is the transition probability function with $\sum_{\Theta \in \mathcal{F}(s,a)} \mathcal{T}(s, a, \Theta) = 1$ for all $s \in S$ and $a \in A(s)$, and $L \colon S \to 2^{\AP}$ is the labeling function.
The main difference between MDPSTs and traditional MDPs is its  representation of \emph{transition uncertainty}.
MDPs map state-action pairs to a probability distribution over all states $S$, while MDPSTs map a state-action pair to a probability distribution over the power set of $S$.
The adversarial environment decides what successor in a subset will be selected and thus forms a \emph{feasible distribution} described below, where $\distr(X)$ denotes the set of probability distributions over the set $X$.
%the pair $(\mathcal{F}, \mathcal{T})$ to map state-action pairs to probability distributions over the \emph{power set}.
% 
%However, since the outcome set $\mathcal{F}(s,a)$ only provides subsets of successor states rather than direct distributions, we next define the notion of \emph{feasible distributions} to connect abstract transitions with executable behavior.

\begin{definition}[Feasible distribution]
\label{def:feasibleDistribution}
Given an MDPST $\mdpst = (S, \bar{s}, A, \mathcal{F}, \mathcal{T}, L)$, $s \in S$, and $a \in A(s)$,  $\mathfrak{h}_{s}^{a} \in \distr(S)$ is a \emph{feasible distribution} of $(s,a)$, denoted by $s \xrightarrow{a} \mathfrak{h}_{s}^{a}$, if there exist $\alpha_{\Theta} \in \distr(\Theta)$ for each $\Theta \in \mathcal{F}(s, a)$ such that $\mathfrak{h}_{s}^{a}(s') = \sum_{\Theta \in \mathcal{F}(s, a)}  \mathcal{T}(s, a, \Theta) \cdot \alpha_{\Theta}(s')$ for each $s' \in S$.
We denote by $\mathfrak{H}_{s}^{a}$ the set of all feasible distributions $\mathfrak{h}_{s}^{a}$ of $(s, a)$.    
\end{definition}

This definition captures all valid ways the system can evolve under uncertain transition sets.
Feasible distributions bridge the gap between set-valued transitions and executable policies.
Each $\alpha_{\Theta}(s')$ represents the (unknown) likelihood of transitioning to $s'$ given $\Theta$ was selected.

Conceptually, $\mathfrak{H}_{s}^{a}$ defines all possible transition distributions consistent with the MDPST's uncertainty model.

A (finite) \emph{path} $\fpath \in (S \times A)^{*} \times S$ of an MDPST $\mdpst$ is an alternating sequence of states and actions ending with a state, $\fpath = s_{0} a_{0} s_{1} \dots s_{N}$, such that for each $0 \leq i < N$, there is $\Theta_{i} \in \mathcal{F}(s_{i}, a_{i})$ such that $\mathcal{T}(s_{i}, a_{i}, \Theta_{i}) > 0$ and $s_{i+1} \in \Theta_{i}$; 
we let $\last(\fpath) = s_{N}$ denote its last state.
Infinite paths are defined similarly.
We let $\fpaths$ ($\ipaths$, resp.) denote the sets of all finite (infinite, resp.) paths of $\mdpst$.

There are two sources of nondeterminism in MDPSTs: 
the choice of the next action in the current state and the choice of the corresponding feasible distribution. 
The former is resolved by a \emph{strategy}, while the latter by a \emph{nature}.

Let $\bot \notin A$ be a fresh symbol we use to denote termination.
We now define a strategy that we look for in a planning problem.
\begin{definition}[Strategy]
\label{def:strategy}
    A function $\strategy \colon S \to \distr(A \cup \{\bot\})$ is a \emph{strategy} if $\strategy(s) \in \distr(A(s) \cup \{\bot\})$ for each $s \in S$.
\end{definition}
In an MDPST, a nature is an adversarial player that tries to prevent the agent from fulfilling its tasks.
\begin{definition}[Nature]
\label{def:nature}
    A function $\nature \colon S \times A \to \distr(S)$ is a \emph{nature} if $\nature(s, a) \in \mathfrak{H}_{s}^{a}$ for each $s \in S$ and $a \in A(s)$.
\end{definition}
We denote by $\strategies_{\mdpst}$ and $\natures_{\mdpst}$ the set of all strategies and natures for $\mdpst$, respectively. 

Given an MDPST $\mdpst$, a strategy $\strategy \in \strategies_{\mdpst}$, a nature $\nature \in \natures_{\mdpst}$, and a state $s$, $\strategy$ and $\nature$ induce a probability measure over the finite paths of $\mdpst$ from the state $s$ as follows:
the basic measurable events are the cylinder sets of finite paths, where the \emph{cylinder set} of a finite path $\fpath$ is the set $\cylinder{\fpath} = \setcond{\fpath' \in \ipaths}{\text{$\fpath$ is a prefix of $\fpath'$}}$.
The probability $\prob_{s}^{\strategy, \nature}$ of the cylinder set $\cylinder{\fpath}$ is defined inductively as follows:
$\prob_{s}^{\strategy, \nature}(\cylinder{s}) = 1$;
$\prob_{s}^{\strategy, \nature}(\cylinder{t}) = 0$ for $t \neq s$; and
$\prob_{s}^{\strategy, \nature}(\cylinder{\fpath}) = \prob_{s}^{\strategy, \nature}(\cylinder{\fpath'}) \cdot \strategy(\last(\fpath'))(a) \cdot \nature(\last(\fpath'), a)(t)$ for $\fpath = \fpath' a t$.
The probability of $\fpath$ is defined as $\prob_{s}^{\strategy, \nature}(\fpath) = \prob_{s}^{\strategy, \nature}(\cylinder{\fpath}) \cdot \strategy(\last(\fpath))(\bot)$.
Standard measure-theoretical arguments ensure that $\prob_{s}^{\strategy, \nature}$ extends uniquely to the $\sigma$-field generated by cylinder sets (see, e.g.,~\cite{Billingsley95}).
We write $\prob^{\strategy, \nature}$ for $\prob_{\bar{s}}^{\strategy, \nature}$ and define $\prob^{\strategy}$ as $\prob^{\strategy} = \min_{\nature \in \natures_{\mdpst}} \prob^{\strategy, \nature}$. 

We say that a strategy $\strategy$ is \emph{terminating} if $\prob^{\strategy}(\fpaths) = 1$.

\section{Consumption MDPSTs}
\label{sec:cmdpsts}

In addition to quantifiable and unquantifiable uncertainties, we now introduce our model called \emph{Consumption MDPSTs} (CMDPSTs) which take into account resource constraints.
This model in fact generalizes consumption MDPs~\cite{Blahoudek2020}, since MDPSTs are more general than MDPs.

\begin{definition}%[Consumption MDPST]
\label{def:cmdps}
A CMDPST is a tuple $\mdpst = (S, \bar{s}, A, \mathcal{F}, \mathcal{T}, L, C, R, \rcap)$
where
\begin{itemize}
  \item $(S, \bar{s}, A, \mathcal{F}, \mathcal{T}, L)$ is an MDPST,
  \item $C \colon S \times A \to \posreals$ is a resource consumption function,
  \item $R \subseteq S$ is the set of reload states where the resource level can be reset to full capacity, and
  \item $\rcap \in \posreals$ is the maximum resource capacity.
\end{itemize}
\end{definition}

The CMDPST framework introduces resource-aware path analysis, where for any finite path $\fpath = s_{0} a_{0} s_{1} \dots s_{N + 1}$, its cumulative resource consumption is computed as $\sum_{i = 0}^{N} C(s_{i}, a_{i})$. 
To support persistent operation in long-horizon tasks, specific states $R \subseteq S$ act as \emph{reload states} where the accumulated cost resets to zero. %, e.g., charging stations. 

A strategy is considered \emph{feasible} if, under all adversarial choices by nature, the cost of every finite path occurring with non-zero probability never exceeds the capacity $\rcap$. 
This ensures that the system remains operational under worst-case uncertainty without violating its resource constraints.
Formally, given a finite path $\fpath = s_{0} a_{0} s_{1} \dots s_{N+1}$, let $Z = \{0, N\} \cup \setcond{0 \leq i \leq N}{s_{i} \in R}$ be the positions where we reset the cost to zero. 
The maximum resource cost of $\fpath$ is
\[
    \cost(\fpath) = \max_{z \in Z} \sum_{i = z}^{\min Z \cap \naturals_{> z}} C(s_{i}, a_{i})
\]
and we call $\fpath$ \emph{feasible} if $\cost(\fpath) \leq \rcap$;
we call a strategy $\strategy$ \emph{feasible} if $\prob^{\strategy}(\setcond{\fpath \in \fpaths}{\text{$\fpath$ is feasible}}) = 1$.
The intuition is that, we can divide the feasible path $\fpath$ into several fragments whose ends are reload/reset states in which the total cost of each fragment must not exceed $\rcap$.

\begin{figure}[t]

    \centering
    \includegraphics[width=0.7\linewidth]{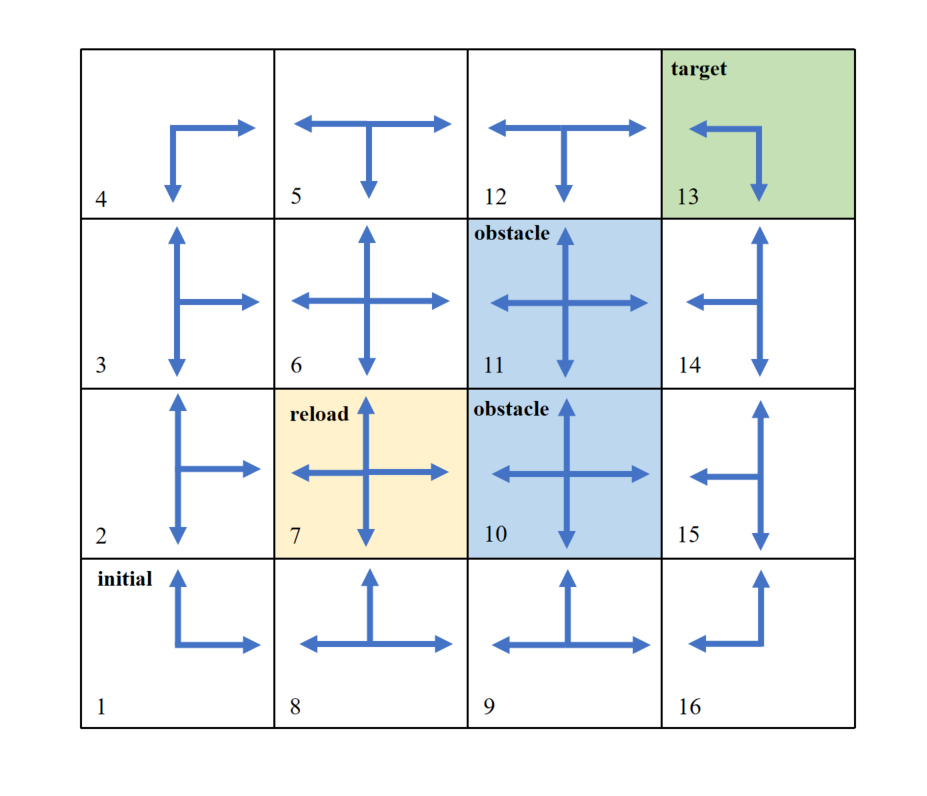}
    % \vspace{-15pt}
    \caption{Simplified Warehouse Planning}
    \label{fig:AGVplanning}
    % \vspace{-20pt} 
\end{figure}
Let us now motivate CMDPSTs for robot planning under uncertainties and resource constraints.
% \begin{example} 
Consider again the warehouse scenario introduced in Section~\ref{sec:motivatingExample} with 16 distinct warehouses arranged in a $4 \times 4$ grid; 
each warehouse corresponds to a state $s \in \mathcal{S} = \{1,2,\dots,16\}$, shown in Fig.~\ref{fig:AGVplanning}.  
Warehouse~7 is designated as a \emph{reload state}, where the AGV may recharge its battery.  
Warehouses~10 and~11 are \emph{obstacle states}, which are inaccessible due to construction.  
The AGV starts at warehouse~1 and must accomplish a delivery task with temporal requirements: 
% it must first visit warehouse~3 to \emph{pick up goods}, and then eventually reach warehouse~13, the designated \emph{delivery target}. 
it must eventually reach warehouse~13, the designated \emph{delivery target} while always \emph{avoiding} warehouses~10 and~11, which are closed.
%\at{Wouldn't be more interesting to have obstacles on 10 and 11 and have the target at 13?}

At each state, the AGV can choose to move according to one of the directions  
$A = \{\texttt{UP}, \texttt{DOWN},\texttt{RIGHT}, \texttt{LEFT}\}$, if possible.
At warehouse $s \in S$, each action $a \in A(s)$ incurs a \text{cost} $C(s, a) \in \posreals$ representing the energy consumed.  
The AGV is equipped with a battery of limited capacity $\rcap$, and its energy decreases according to $C(s,a)$.  
When the AGV reaches a reload warehouse, its battery is recharged to $\rcap$.  

The AGV's motion is modeled as a CMDPST, incorporating both stochastic uncertainty and adversarial nondeterminism (e.g., human operator intervention).  
The stochastic transition function $\mathcal{T}$ is defined as follows.  
When the AGV attempts to move to a desired warehouse $i$ under an action:
\begin{itemize}
    \item with probability $0.8$, the action succeeds, and the AGV transitions to warehouse $i$;
    \item with probability $0.2$, the action underperforms due to actuation noise, and the AGV transitions to warehouse $i{-}1$, if $i > 1$.
\end{itemize}
To model environmental (e.g., human intervention) uncertainty, the set-valued transition function $\mathcal{F}$ includes an additional nondeterministic option.  
For instance, the environment may force the AGV into warehouse $i{-}2$, regardless of the intended action, provided $i > 2$.
As a concrete example, let the AGV be at warehouse~2 with $5$ units of remaining energy and consider the \texttt{RIGHT} action having $C(2, \texttt{RIGHT}) = 2$.  
The intended target is warehouse~7 and the possible successor sets and their probabilities are
% \begin{itemize}
% \item 
    $\mathcal{F}(2, \texttt{RIGHT}) = \{ \{7, 5\}, \{6, 5\} \}$ and
% \item 
    $\mathcal{T}(2, \texttt{RIGHT}, \{7,5\}) = 0.8$ and 
    $\mathcal{T}(2, \texttt{RIGHT}, \{6, 5\}) = 0.2$.
% \end{itemize}
If the AGV moves to warehouse~7, the reload rule applies and its energy level is reset to $\rcap$;  
otherwise, the remaining energy decreases to $3$.  
% \end{example}

Our goal is to synthesize a feasible strategy (i.e., the AGV never runs out of energy) and optimally robust (i.e., it maximizes the probability of delivering the goods independently on how humans interfere with the AGV).

\begin{definition}[$\beta$-robust strategy]
    Given a CMDPST $\mdpst$, an $\ltlf$ formula $\varphi$, and a threshold $\beta$, we say that a strategy $\strategy \in \strategies_{\mdpst}$ is \emph{$\beta$-robust for $\varphi$} if it is feasible and $\prob^{\strategy}([\varphi]) \geq \beta$.

    A $\beta$-robust strategy $\strategy$ is \emph{optimal} for $\varphi$ if $\prob^{\strategy}([\varphi]) \geq \prob^{\strategy'}([\varphi])$ for all $\beta$-robust strategies $\strategy'$ for $\varphi$.
\end{definition}

\paragraph*{Problem formulation}
To synthesize an optimal robust strategy that ensures the satisfaction of a high-level task specified by an $\ltlf$ formula $\varphi$ while guaranteeing that the resource is never exhausted under adversarial environment behaviors, we formulate the \emph{Resource Constrained Optimal Robust Strategy Synthesis} problem:
given a CMDPST model $\mdpst$ and an $\ltlf$ specification $\varphi$, the objective is to compute a \emph{feasible} strategy (i.e., the resource never runs out under all adversarial choices by nature) such that, the probability of satisfying the $\ltlf$ specification $\varphi$ is maximized.

\section{Na\"ively Constrained Optimal Robust Strategy Synthesis}
\label{sec:unroll}

We first introduce a na\"ive method for solving the resource constrained optimal robust strategy synthesis problem given a CMDPST $\mdpst$ and an $\ltlf$ formula $\varphi$ over $\AP$.
The basic idea is to convert the CMDPST to an MDPST so we can use a well-developed approach to synthesize an optimal robust strategy from an MDPST.
The algorithm is given as follows:
\begin{enumerate}
    \item construct a DFA $\dfa_{\varphi}$ from $\varphi$ such that $\lang(\dfa_{\varphi}) = [\varphi]$,
    \item build the product CMDPST $\mdpst^{\times}$ of $\mdpst$ and $\dfa_{\varphi}$,
    \item unroll $\mdpst^{\times}$ into a MDPST $\mdpst^{\unrolled}$, and
    \item compute the optimal robust strategy on $\mdpst^{\unrolled}$.
\end{enumerate}

Note that it is possible to combine steps 2) and 3) together without constructing the intermediate product CMDPST $\mdpst^{\times}$ explicitly.
However, in order to make our presentation simpler and to further include optimizations, we will introduce the construction step by step.

\paragraph*{DFA construction}
The first $\ltlf$-to-DFA conversion was proposed in~\cite{degiacomo2013linear} and the current state of the art conversions normally employ a compositional methodology~\cite{DBLP:conf/aaai/BansalLTV20,DBLP:conf/aips/GiacomoF21,DBLP:conf/fmcad/BansalKL24}.
In our method, we use the off-the-shelf tool from~\cite{DBLP:conf/aips/GiacomoF21}.
Assume that the constructed DFA is $\dfa_{\varphi} = (\Sigma = 2^{\AP}, Q, \bar{q}, \delta, \acc)$ with $\lang(\dfa_{\varphi}) = [\varphi]$.

\paragraph*{Product $\mdpst^{\times}$ construction}
Given a CMDPST $\mdpst = (S, \bar{s}, A, \mathcal{F}, \mathcal{T}, L, C, R, \rcap)$ and the DFA $\dfa_{\varphi}$, the \emph{product CMDPST} $\mdpst^{\times}$ of $\mdpst$ and $\dfa_{\varphi}$ is defined as
% \vspace{-5pt} 
% \[
%     \mdpst^{\times} = (S^{\times}, \bar{s}^{\times}, A^{\times}, \mathcal{F}^{\times}, \mathcal{T}^{\times}, L^{\times}, C^{\times}, R^{\times}, \rcap^{\times}, \acc^{\times})
% \]
the CMDPST $\mdpst^{\times} = (S^{\times}, \bar{s}^{\times}, A^{\times}, \mathcal{F}^{\times}, \mathcal{T}^{\times}, L^{\times}, C^{\times}, R^{\times}, \rcap^{\times}, \acc^{\times})$
% \vspace{-5pt} 
where
\begin{itemize}
\item 
    $S^{\times} = S \times Q$ and $\bar{s}^{\times} = (\bar{s}, \bar{q})$;
\item
    $A^{\times} = A$;
\item
    $\mathcal{F}^{\times}$ and $\mathcal{T}^{\times}$ are defined as follows:
    for each $(s, q) \in S^{\times}$ and $a\in A^{\times}$, let $q' = \delta(q, L(s))$. 
    Then $\mathcal{F}^{\times}((s, q), a) = \setcond{\Theta \times \{q'\}}{\Theta \in \mathcal{F}(s, a)}$ and, for each $\Theta \in \mathcal{F}(s, a)$, $\mathcal{T}^{\times}((s, q), a, \Theta \times \{q'\}) = \mathcal{T}(s, a, \Theta)$;
\item 
    $L^{\times}((s, q)) = L(s)$ for each $(s, q) \in S^{\times}$;
\item 
    $C^{\times}((s, q), a) = C(s, a)$;
\item 
    $R^{\times} = \setcond{(s,q) \in S^{\times}}{s \in R}$;
\item 
    $\rcap^{\times} = \rcap$; and
\item 
    $\acc^{\times} = \setcond{(s, q) \in S^{\times}}{q \in \acc}$.
\end{itemize}

By building the product $\mdpst^{\times}$, it is easy to identify which paths in $\mdpst$ generate traces in $[\varphi]$, since every path in $\mdpst^{\times}$ that can reach a final state in $\acc^{\times}$ is one of such paths.
This is, however, not enough to derive a feasible strategy as we still need to consider the resource constraints in the $\mdpst^{\times}$.
To this end, we can convert the CMDPST $\mdpst^{\times}$ to a MDPST $\mdpst^{\unrolled}$, where a well-developed optimal robust strategy approach (e.g.,~\cite{DBLP:conf/ijcai/Yu0GKV24}) can be exploited.
The conversion approach is quite simple yet effective: we record in each state also its current resource level.
We call this conversion approach the \emph{unroll} approach.

\paragraph*{Unrolled $\mdpst^{\unrolled}$ construction}
For the input CMDPST $\mdpst^{\times} = (S^{\times}, \bar{s}^{\times}, A^{\times}, \mathcal{F}^{\times}, \mathcal{T}^{\times}, L^{\times}, C^{\times}, R^{\times}, \rcap^{\times}, \acc^{\times})$, we now define the \emph{unrolled product MDPST} $\mdpst^{\unrolled}$ as
% \vspace{-4pt} 
% \[
%     \mdpst^{\unrolled} = (S^{\unrolled}, \bar{s}^{\unrolled}, A^{\unrolled}, \mathcal{F}^{\unrolled}, \mathcal{T}^{\unrolled}, L^{\unrolled}, \acc^{\unrolled})
% \]
the MDPST $\mdpst^{\unrolled} = (S^{\unrolled}, \bar{s}^{\unrolled}, A^{\unrolled}, \mathcal{F}^{\unrolled}, \mathcal{T}^{\unrolled}, L^{\unrolled}, \acc^{\unrolled})$
%C^{\unrolled}, R^{\unrolled}, \rcap^{\unrolled})
where
\begin{itemize}
\item 
    $S^{\unrolled} \subseteq S^{\times} \times [0, \rcap]$ is the smallest set such that $\bar{s}^{\unrolled} = (\bar{s}^{\times}, \rcap) \in S^{\unrolled}$ and that is closed under the functions $\mathcal{F}^{\unrolled}$ and $\mathcal{T}^{\unrolled}$ defined below,
\item 
    $A^{\unrolled}((s, c)) = A^{\times}(s)$ for all $(s, c) \in S^{\unrolled}$,
\item 
    $\mathcal{F}^{\unrolled}$ and $\mathcal{T}^{\unrolled}$ are constructed as follows:
    for all $(s, c) \in S^{\unrolled}$ and $a \in A^{\times}(s)$ such that $C^{\times}(s, a) \leq c$ and for each $\Theta \in \mathcal{F}^{\times}(s, a)$, let
    $\Theta^{\unrolled} = \setcond{(s', \rcap)}{s' \in \Theta \cap R^{\times}} \cup \setcond{(s', c - C^{\times}(s, a))}{s' \in \Theta \setminus R^{\times}}$;
    then $\Theta^{\unrolled} \in \mathcal{F}^{\unrolled}((s, c), a)$ and $\mathcal{T}^{\unrolled}((s, c), a, \Theta^{\unrolled}) = \mathcal{T}^{\times}(s, a, \Theta)$,
\item 
    $L^{\unrolled}((s, c)) = L^{\times}(s)$ for all $(s, c) \in S^{\unrolled}$, and

\item 
    $\acc^{\unrolled} = \setcond{(s, c) \in S^{\unrolled}}{s \in \acc^{\times}}$.
\end{itemize}

This construction preserves the original transition probabilities while explicitly encoding current resource level into the state space.  
The resource level $c$ in each state $(s, c)$ of $S^{\unrolled}$ is positive, and this will guarantee that every path in $\mdpst^{\unrolled}$ is feasible.

\begin{proposition}
\label{prop:unrollingCorrectness}
    Given a product CMDPST $\mdpst^{\times}$, let $\mdpst^{\unrolled}$ be its unrolled MDPST.
    Then, the following properties hold:
    \begin{enumerate}
    \item 
        for every feasible strategy $\strategy^{\times} \in \strategies_{\mdpst^{\times}}$ there is a strategy $\strategy^{\unrolled} \in \strategies_{\mdpst^{\unrolled}}$ such that for every feasible path $\fpath^{\times} \in \fpaths_{\mdpst^{\times}}$ there is a path $\fpath^{\unrolled} \in \fpaths_{\mdpst^{\unrolled}}$ such that $\proj(\fpath^{\unrolled}) = \fpath^{\times}$ and $\prob^{\strategy^{\unrolled}}_{\mdpst^{\unrolled}}(\fpath^{\unrolled}) = \prob^{\strategy^{\times}}_{\mdpst^{\times}}(\fpath^{\times})$;
    \item 
        for every strategy $\strategy^{\unrolled} \in \strategies_{\mdpst^{\unrolled}}$, there is a feasible strategy $\strategy^{\times} \in \strategies_{\mdpst^{\times}}$ such that for every path $\fpath^{\unrolled} \in \fpaths_{\mdpst^{\unrolled}}$ we have that $\prob^{\strategy^{\times}}_{\mdpst^{\times}}(\proj(\fpath^{\unrolled})) = \prob^{\strategy^{\unrolled}}_{\mdpst^{\unrolled}}(\fpath^{\unrolled})$,
    \end{enumerate}
    where $\proj((s_{0}, c_{0}) a_{0} \dots (s_{n}, c_{n})) = s_{0} a_{0} \dots s_{n}$.
\end{proposition}
\begin{proof}
    Both properties are proved by equating $\strategy^{\times}(s^{\times})$ and $\strategy^{\unrolled}((s^{\times}, c))$; 
    for the first property, the path $\fpath^{\unrolled}$ is obtained from $\fpath^{\times} = s^{\times}_{0} a_{0} \dots s^{\times}_{n}$ as $\fpath^{\unrolled} = (s^{\times}_{0}, c_{0}) a_{0} \dots (s^{\times}_{n}, c_{n})$ where $c_{0} = \rcap$ and for each $0 \leq i < n$, $c_{i+1} = \rcap$ if $s^{\times}_{i + 1} \in R^{\times}$ and $c_{i+1} = c_{i} - C(s_{i}, a_{i})$ otherwise.
    A direct application of the definitions yields the stated properties.
\end{proof}

We now extend the warehouse transportation example from Fig.~\ref{fig:Warehouse} to incorporate \emph{resource consumption} and the \emph{unrolling construction}.  
Each warehouse $s \in \mathcal{S} = \{1, \dots, 16\}$ is now augmented with a resource level $c \in [0,\rcap]$, resulting in unrolled states of the form $(s,c)$.  
The AGV starts at $(1,\rcap)$, i.e., at warehouse~1 with the battery fully charged.

At each step, the AGV chooses action $a$ among the actions $A = \{\texttt{UP}, \texttt{DOWN}, \texttt{LEFT}, \texttt{RIGHT}\}$.  
Executing $a$ at state $(s,c)$ consumes $C(s,a)$ units of resource, leading to transitions
$(s,c) \xrightarrow{a} (s', c-C(s,a))$ for each $s' \in \Theta \in \mathcal{F}(s,a)$ when $C(s,a) \leq c$;
when $C(s,a) > c$, action $a$ is disabled.  

Reload warehouses allow resource replenishment: since warehouse~7 is a \emph{reload state}, any state $(7,c)$ reached is reset to $(7,\rcap)$.  
Obstacle warehouses (i.e., warehouses~10 and~11) remain prohibited in the unrolled model, and the labeling function $L^{\unrolled}$ is inherited from the base CMDPST.

As a concrete example, suppose the AGV is at unrolled state $(2, 5)$ and executes $\texttt{RIGHT}$ with $C(2, \texttt{RIGHT}) = 2$.  
In the original model, the intended target is warehouse~7.  
In the unrolled model, the possible successor sets are $\{(7, \rcap), (5, 3)\}$ with probability $0.8$ and $\{(6, 3), (5, 3)\}$ with probability $0.2$.
Here, the remaining resource decreases from $5$ to $3$ by the cost $C(2, \texttt{RIGHT}) = 2$, except when reaching warehouse~7, where it is reset to $(7, \rcap)$ due to the reload mechanism.

\paragraph*{Optimal robust strategy extraction}
Now, we will introduce the standard approach for MDPSTs to extract an optimal robust strategy.
Let $\mdpst^{\unrolled} = (S^{\unrolled}, \bar{s}^{\unrolled}, A^{\unrolled}, \mathcal{F}^{\unrolled}, \mathcal{T}^{\unrolled}, L^{\unrolled}, \acc^{\unrolled})$ be an unrolled MDPST.
To make our presentation more general, we also assume that we are given a winning region $\acc^{\unrolled} \subseteq W \subseteq S^{\unrolled}$ in which the agent has a strategy to reach almost surely target states $\acc^{\unrolled}$ against any nature.
Of course, we can always let $W = \acc^{\unrolled}$ and the following procedure will still be correct.

To obtain the optimal robust strategy, we first define a \emph{robust value function} $V \colon S^{\unrolled} \to [0,1]$ defined as:
% \vspace{-4pt} 
\[
    V(s) = \max_{\strategy} \min_{\nature} \prob^{\strategy, \nature}_{s}(\lfinally W)
\]
where $\strategy$ is the agent’s strategy and $\nature$ is the nature’s strategy (resolving uncertainty).

In order to compute the robust value function, we resort to solving a \emph{robust Bellman equation system}.
More precisely, for all $s \in W$, we set $V(s) = 1$. 
For $s \in S^{\unrolled} \setminus W$, $V(s)$ needs to satisfy the following constraint:
% \vspace{-3pt} 
\[
    V(s) = \max_{a \in A^{\unrolled}(s)} \sum_{\Theta \in \mathcal{F}^{\unrolled}(s,a)} \mathcal{T}^{\unrolled}(s, a, \Theta) \cdot \min_{s' \in \Theta} V(s')
\]

This definition captures the worst-case reachability probability under both adversarial nondeterminism and probabilistic uncertainty.

After computing the robust value function $V$, we are in fact able to record the action that leads to the maximal value in $V$.
Indeed, it is easy to obtain an optimal robust strategy $\strategy^{\unrolled} \colon S^{\unrolled} \to A^{\unrolled} \cup \{\bot\}$ on $\mdpst^{\unrolled}$ where $\strategy^{\otimes}(s)$ is the action in $A^{\unrolled}$ that leads to the maximal robust value $V(s)$ if $V(s) > 0$, or just $\strategy^{\otimes}(s) = \bot$ when $V(s) = 0$.
Note that, for reachability goals, there are history-independent optimal strategies on MDPSTs~\cite{DBLP:conf/ijcai/Yu0GKV24}. 

In order to obtain an optimal robust strategy on the original CMDPST $\mdpst$, we will use the resource level and the DFA state encoded in a state $s \in S^{\unrolled}$ as part of the memory.
Thus, we will obtain a history-dependent optimal robust strategy $\strategy$ on $\mdpst$.
The strategy $\strategy \colon S \times M \to A$ where $M = Q \times [0, \rcap] $ and the execution of $\strategy$ is defined as follows:
\begin{itemize}
\item 
    The initial memory state of the strategy is $\bar{m} = \langle \bar{q}, \rcap \rangle$.
\item 
    For a state $s \in S$ and a current memory state $m = \langle q, c\rangle$, we have $\strategy(s, m) = \strategy^{\unrolled}(s, q, c)$ and then the memory state $m'$ is updated to $\langle \delta(q, L(s)), c - C(s, a)\rangle$.
\item 
    When $m = \langle q, c\rangle$ involves a final state $q \in \acc$ in the DFA $\dfa_{\varphi}$, we let $\strategy(s, m) = \bot$.
    This means that the agent has fulfilled the task and decides to terminate.
\end{itemize}

\begin{theorem}
\label{thm:optimalStrategySynthesis}
    Our synthesized strategy $\strategy$ is an optimal robust strategy for $\mdpst$ to fulfill the $\ltlf$ specification $\varphi$.
\end{theorem}
\begin{proof}
    The proof is based on a direct application of the definitions and Proposition~\ref{prop:unrollingCorrectness} to equate the probability values between feasible paths of $\mdpst^{\times}$ and paths of $\mdpst^{\unrolled}$.
\end{proof}

\section{Optimized Robust Strategy Synthesis}
\label{sec:optimizedStrategySynthesis}

In the previous section, we introduced the unrolled approach to convert a CMDPST $\mdpst^{\times}$ to an MDPST $\mdpst^{\unrolled}$ which records the current resource level in each state.
This na\"ive approach builds an enormous state space for the MDPST and thus gives an optimal robust strategy of large size.
To this end, we propose to first prune the CMDPST $\mdpst^{\times}$ before converting to the MDPST.
Recall that the size of our synthesized strategy is proportional to that of the MDPST.
Therefore, our optimization is to remove all states from which the agent does not have a strategy maintaining feasibility.   

We observe that the approach we present in this section can also be used in the design phase of the system to check whether there exists a feasible strategy given a CMDPST $\mdpst$ and a set $T$ of target states.
That is, when deploying a system to complete a task, it is often necessary to check first whether the current system has a feasible strategy to possibly reach the target states instead of computing an optimal robust strategy directly.
Any deployment error found in this phase can be fixed immediately.
For instance, once we find out that there are no feasible strategies, we can look for the root causes, such as that there are not enough charging stations or their spatial distribution is not reasonable, so we can add more or adjust the locations of the charging stations in order to make it possible to fulfill the given (set of) task(s).

Another advantage of our proposed optimization is that it only involves simple graph exploration and value accounting.
Thus, it is easy to implement.
The central idea in our optimization is to compute the feasible region where the agent has a strategy that keeps it moving against any nature.
Then, we can just check whether some target state is in the feasible region.
For pruning, we only need to remove all states and transitions that are outside the feasible region.

\subsubsection{Feasible region computation}

\begin{algorithm}[t]
\caption{\proc{FeasibleRegion}}
\label{alg:feasibleRegion}
    \begin{algorithmic}[1]
        \Require CMDPST $\mdpst = (S, \bar{s}, A, \mathcal{F}, \mathcal{T}, L, C, R, \rcap)$%
        \State initialize $\Vec{a} \in A^{S}$ to be $\emptyset$ for each state s \label{alg:feasibleRegion:initA}%
        \State initialize $\Vec{c} \in \posreals^{S \cup \{\text{\textvisiblespace}\}}$ to be $\infty$ for each state s \label{alg:feasibleRegion:initC}%
        \If{$\bar{s} \in R$} \label{alg:feasibleRegion:reloadInitial}%
            % \State 
            $\Vec{c}(\bar{s}) \gets 0$ \label{alg:feasibleRegion:reloadInitial:updateC}%
        \EndIf%
        \State $S_{f} \gets \emptyset$; $E \gets \emptyset$; $s \gets \bar{s}$; $c \gets 0$ \label{alg:feasibleRegion:initVariables}%
        \Repeat \label{alg:feasibleRegion:mainLoopStart}%
            \State $S_{f} \gets S_{f} \cup \{s\}$ \label{alg:feasibleRegion:updateSf}%
            \ForAll{$a \in A(s)$ such that $c + C(s, a) \leq \rcap$} \label{alg:feasibleRegion:actionLoopStart} \label{alg:feasibleRegion:actionLoopGuard}%
                \State $\Vec{a}(s) \gets \Vec{a}(s) \cup \{a\}$ \label{alg:feasibleRegion:updateAs}%
                \ForAll{$s' \in \cup_{\Theta \in \mathcal{F}(s, a)} \Theta$} \label{alg:feasibleRegion:successorsLoopStart} \label{alg:feasibleRegion:successorsLoopGuard}%
                    \If{$\Vec{c}(s') > c + C(s,a)$} \label{alg:feasibleRegion:betterCost}%
                        \State $\Vec{c}(s') = c + C(s, a)$ \label{alg:feasibleRegion:updateCsp}%
                        \If{$s' \in R$} \label{alg:feasibleRegion:reloadCheck}%
                            % \State 
                            $\Vec{c}(s') \gets 0$ \label{alg:feasibleRegion:reloadCheck:updateCsp}%
                        \EndIf%
                        \State $E \gets E \cup \{s'\}$ \label{alg:feasibleRegion:updateE}%
                    \EndIf%
                \EndFor%
            \EndFor \label{alg:feasibleRegion:actionLoopEnd}%
            \State $E, s \gets \proc{PickOne}(E)$ \label{alg:feasibleRegion:pickNext}%
            \State $c \gets \Vec{c}(s)$ \label{alg:feasibleRegion:pickNextCost}%
        \Until{$s =\text{\textvisiblespace}$} \label{alg:feasibleRegion:mainLoopGuard} \label{alg:feasibleRegion:mainLoopEnd}%
        \State \Return $S_{f}, \Vec{a}$ \label{alg:feasibleRegion:return}%
    \end{algorithmic}
 \end{algorithm}

The \proc{FeasibleRegion} algorithm is shown as Algorithm~\ref{alg:feasibleRegion} and, on input a CMDPST $\mdpst$, it returns the set of states $S_{f}$ that are reachable from $\bar{s}$ by a feasible path, as well as the map $\Vec{a}$ that assigns to each state~$s$ the set of actions that preserve feasibility.
To compute $S_{f}$ and $\Vec{a}$, \proc{FeasibleRegion} makes use of two variables~$s$ and~$c$, keeping the state currently under evaluation and its cost, and two additional data structures: 
a set $E$ to remember the reached states that need to be evaluated and a vector $\Vec{c}$ that stores the minimum cost to reach each state from either the initial state or a reload state already in $S_{f}$ or $E$.
The computation proceeds as follows: 
starting from the initial state $s = \bar{s}$ and cost $c = 0$ (Line~\ref{alg:feasibleRegion:initVariables}), we add it to $S_{f}$ (Line~\ref{alg:feasibleRegion:updateSf}) and consider each action $a \in A(s)$ that preserves feasibility, i.e., $c + C(s, a) \leq \rcap$ (Line~\ref{alg:feasibleRegion:actionLoopGuard}). 
Since feasibility is preserved, we add it to $\Vec{a}(s)$ (Line~\ref{alg:feasibleRegion:updateAs}) and for each possible $a$-successor $s' \in \cup_{\Theta \in \mathcal{F}(s, a)} \Theta$ (Line~\ref{alg:feasibleRegion:successorsLoopGuard}), we check whether we have a lower cost to reach it (Line~\ref{alg:feasibleRegion:betterCost}); 
if this is the case, then we update its cost (Line~\ref{alg:feasibleRegion:updateCsp}) to $\Vec{c}(s') = c + C(s, a)$ or to $0$ if it is a reload state (Lines~\ref{alg:feasibleRegion:reloadCheck}) and add it to $E$ for further analysis (Line~\ref{alg:feasibleRegion:updateE}), as some action in $A(s')$ might have now become feasible.
Once we have dealt with all actions and successor states, we take the next state to be evaluated from $E$ (Line~\ref{alg:feasibleRegion:pickNext}) by calling $\proc{PickOne}(E)$.
This auxiliary procedure removes one arbitrary element from $E$ and returns it together with the updated $E$;
if no such element exists, then $(\emptyset, \text{\textvisiblespace})$ is returned, where $\text{\textvisiblespace} \notin S$ is a fresh symbol.
We then update the current cost $c$ of $s$ (Line~\ref{alg:feasibleRegion:pickNextCost}) and repeat these steps until we have no other state to evaluate (Line~\ref{alg:feasibleRegion:mainLoopGuard}) so we can return $S_{f}$ and $\Vec{a}$ (Line~\ref{alg:feasibleRegion:return}).

The fact that \proc{FeasibleRegion} terminates follows from the fact that $S$ is finite and that every time a state $s$ is added to $E$ (Line~\ref{alg:feasibleRegion:updateE}) its cost is strictly decreased (cf.~Lines~\ref{alg:feasibleRegion:betterCost}-\ref{alg:feasibleRegion:reloadCheck:updateCsp}) and the cost cannot be negative (cf.~Definition~\ref{def:cmdps}), so $s$ can be added to $E$ only a finite number of times.

It is easy to observe the following properties about $S_{f}$ and $\Vec{a}$ returned by \proc{FeasibleRegion}:
\begin{lemma}
\label{lem:alg:feasibleRegion:Properties}
    Given a CMDPST $\mdpst$, let $(S_{f}, \Vec{a}) = \proc{FeasibleRegion}(\mdpst)$.
    Then the following properties hold:
    \begin{enumerate}
    \item 
        $S_{f} \subseteq S$ and for each $s \in S_{f}$, $\Vec{a}(s) \subseteq A(s)$.
    \item
        For each $s \in S_{f}$, $a \in \Vec{a}(s)$, and $\Theta \in \mathcal{F}(s, a)$, $\Theta \subseteq S_{f}$.
    \item
        For each feasible path $\fpath \in \fpaths$, $\last(\fpath) \in S_{f}$.
    \item
        For each $s \in S_{f}$, there exists a feasible path $\fpath \in \fpaths$ such that $\last(\fpath) = s$.
    \item
        For each $s \in S \setminus S_{f}$, there is no path $\fpath \in \fpaths$ such that $\last(\fpath) = s$ that is feasible.
    \item
        For each $s \in S_{f}$ and each feasible path $\fpath \in \fpaths$ such that $\last(\fpath) = s$, for each $a \in \Vec{a}(s)$ and each $s' \in \cup_{\Theta \in \mathcal{F}(s, a)} \Theta$, the path $\fpath a s'$ is feasible.
    \item
        For each $s \in S_{f}$ and each feasible path $\fpath \in \fpaths$ such that $\last(\fpath) = s$, for each $a \in A(s) \setminus \Vec{a}(s)$ and each $s' \in \cup_{\Theta \in \mathcal{F}(s, a)} \Theta$, the path $\fpath a s'$ is not feasible.        
    \end{enumerate}
\end{lemma}
\begin{proof}
    The proof is based on a simple inspection of \proc{FeasibleRegion}'s code and on specific invariants for the outer loop (Lines~\ref{alg:feasibleRegion:mainLoopStart}-\ref{alg:feasibleRegion:mainLoopEnd}) like ``for all $s \in S_{f} \cup E$ there is a feasible path $\fpath \in \fpaths$ such that $\last(\fpath) = s$''
    %, ``for all $s \in R \cap (S_{f} \cup E)$, $\Vec{c}(s) = 0$'', 
    or ``for all $s \in S_{f} \cup E$, $\Vec{c}(s)$ equals the cost of a feasible path ending in $s$ that either starts in $\bar{s}$ or in a reload state in $S_{f}$ and no reload state occurs in its interior sub-path". 
\end{proof}

\subsubsection{Target feasibility checking}

Given a CMDPST $\mdpst$ and a set of target states $T$, we can easily verify whether $T$ can be reached by some feasible path by first computing $(S_{f}, \Vec{a}) = \proc{FeasibleRegion}(\mdpst)$ and then checking whether $T \cap S_{f} \neq \emptyset$.
Similarly, we can verify whether all states in $T$ can be reached by some feasible path by checking $T \subseteq S_{f}$.
Note however that even if $T \subseteq S_{f}$ holds, there is no guarantee that $\prob^{\strategy}(\lfinally T) = 1$ for some feasible strategy $\strategy$:
there can still be some nature preventing $T$ to be reached with probability $1$.

\subsubsection{Feasibility pruning}

Pruning a CMDPST to its feasible region is rather easy: 
it is enough to remove all states not in $S_{f}$ and all transitions from the states $s$ in $S_{f}$ whose action is not in $\Vec{a}(s)$;
by Lemma~\ref{lem:alg:feasibleRegion:Properties}, all paths in the pruned CMDPST are feasible.
Formally, given a CMDPST $\mdpst = (S, \bar{s}, A, \mathcal{F}, \mathcal{T}, L, C, R, \rcap)$ and $(S_{f}, \Vec{a}) = \proc{FeasibleRegion}(\mdpst)$, the pruned CMDPST $\mdpst^{\pruned} = (S^{\pruned}, \bar{s}^{\pruned}, A^{\pruned}, \mathcal{F}^{\pruned}, \mathcal{T}^{\pruned}, L^{\pruned}, C^{\pruned}, R^{\pruned}, \rcap^{\pruned})$ has 
\begin{itemize}
\item
    $S^{\pruned} = S_{f}$ and $\bar{s}^{\pruned} = \bar{s}$;
\item
    $A^{\pruned} = \cup_{s \in S_{f}} \Vec{a}(s)$;
\item
    $\mathcal{F}^{\pruned}$ is defined as $\mathcal{F}^{\pruned}(s, a) = \mathcal{F}(s, a)$ only for each $s \in S^{\pruned}$ and $a \in \Vec{a}(s)$;
\item
    $\mathcal{T}^{\pruned}$ is defined as $\mathcal{T}^{\pruned}(s, a, \Theta) = \mathcal{T}(s, a, \Theta)$ only for each $s \in S^{\pruned}$, $a \in \Vec{a}(s)$, and $\Theta \in \mathcal{F}^{\pruned}(s, a)$;
\item
    $L^{\pruned}$ is defined as $L^{\pruned}(s) = L(s)$ only for each $s \in S^{\pruned}$;
\item
    $C^{\pruned}$ is defined as $C^{\pruned}(s, a) = C(s, a)$ only for each $s \in S^{\pruned}$ and $a \in A^{\pruned}(s)$;
\item 
    $R^{\pruned} = R \cap S_{f}$; and
\item 
    $\rcap^{\pruned} = \rcap$.
\end{itemize}

Similarly to CMDPST unrolling, we can prove the following correctness result about feasibility pruning:
\begin{proposition}
\label{prop:pruningCorrectness}
    Given a CMDPST $\mdpst$, let $\mdpst^{\pruned}$ be its feasibility pruned CMDPST.
    % Then, the following properties hold:
    Then, we have that
    \begin{enumerate}
    \item 
        for every feasible strategy $\strategy \in \strategies_{\mdpst}$ there exists a strategy $\strategy^{\pruned} \in \strategies_{\mdpst^{\pruned}}$ such that for every feasible path $\fpath \in \fpaths_{\mdpst}$ we have $\prob^{\strategy^{\pruned}}_{\mdpst^{\pruned}}(\fpath) = \prob^{\strategy}_{\mdpst}(\fpath)$;
    \item 
        for every strategy $\strategy^{\pruned} \in \strategies_{\mdpst^{\pruned}}$, there exists a feasible strategy $\strategy \in \strategies_{\mdpst}$ such that for every path $\fpath^{\pruned} \in \fpaths_{\mdpst^{\pruned}}$ we have $\prob^{\strategy}_{\mdpst}(\fpath^{\pruned}) = \prob^{\strategy^{\pruned}}_{\mdpst^{\pruned}}(\fpath^{\pruned})$.
    \end{enumerate}
\end{proposition}
\begin{proof}
    Both properties can be proved by using the same strategy on both CMDPSTs and then by a direct application of the definitions and Lemma~\ref{lem:alg:feasibleRegion:Properties}.
\end{proof}

Proposition~\ref{prop:pruningCorrectness} ensures that pruning does not alter the maximum value we can obtain in solving the optimal robust strategy synthesis problem, similarly to Proposition~\ref{prop:unrollingCorrectness} and its use in the proof of Theorem~\ref{thm:optimalStrategySynthesis}.

\section{Experiments}
\label{sec:experiments}

In this section, we evaluate our approach on an extended AGV transportation scenario. 
Compared with the motivating example in Section~\ref{sec:motivatingExample}, the experimental setting is enriched in two aspects. 
First, instead of a single AGV, we consider multiple AGVs navigating the warehouse network simultaneously, which increases the interaction between agents and the complexity of the strategy space. AGVs operate in the same environment and interactions are handled by disabling actions that would lead to collisions, i.e., multiple agents cannot occupy the same warehouse simultaneously.
Second, the action model is augmented: in addition to the movements considered before, AGVs are now allowed to traverse alternative connections between warehouses (e.g., shortcuts or cross-aisle passages), which leads to a larger set of possible transitions. 
These extensions make the experiments more representative of realistic logistics scenarios and provide a more demanding benchmark for testing both the performance and accuracy of the proposed methods. 
To systematically assess performance, we vary the warehouse network size from 4 to 16 and report the resulting state space, runtime, and synthesis outcomes.

All simulations are carried out on a laptop with an Intel Core i5-13500HX fourteen-core processor, 16GB DDR5 RAM, and the implementation code can be found at: https://github.com/yihaoyin/CMDPST.git .

In the baseline 4-warehouse scenario, the original MDPST consists of 26 states and 110 transitions. 
After applying the unrolling construction given in Section~\ref{sec:unroll}, the model expands to 286 states and 1846 transitions. 
When further combined with the DFA encoding the $\ltlf$ formula $\neg W_{o} \luntil W_{d}$, which requires to eventually reach the target delivery warehouse $W_{d}$ while avoiding any obstacle warehouse $W_{o}$, 
the resulting product model contains 858 states and 4014 transitions. 
On this product, the optimal robust strategy extraction yields a reachability probability of 0.883 with a computation time of 0.52 seconds. 
We then apply the state-space pruning technique proposed in Section~\ref{sec:optimizedStrategySynthesis}, which eliminates states and transitions that are not feasible. 
This reduction shrinks the product to 275 states and 1743 transitions, leading to a much shorter time of 0.19 seconds while obtaining the same reachability probability value (0.883) as expected. 

\begin{table}[t]
\caption{CMDPST Optimal Strategy Synthesis}
\label{tab:experiments}
\centering
\setlength{\tabcolsep}{1.2mm}
\begin{tabular}{r|rr|rr|rr|rr|l}
\hline
\multicolumn{1}{c|}{Ware-} & \multicolumn{2}{c|}{Time (s)} & \multicolumn{2}{c|}{\# States} & \multicolumn{2}{c|}{\# Transitions} & \multicolumn{1}{c}{Prob-} \\
\multicolumn{1}{c|}{houses} & \multicolumn{1}{c}{na\"ive} & \multicolumn{1}{c|}{pruned} & \multicolumn{1}{c}{na\"ive} & \multicolumn{1}{c|}{pruned} & \multicolumn{1}{c}{na\"ive} & \multicolumn{1}{c|}{pruned} & \multicolumn{1}{c}{ability}\\
\hline
\hline
4 & 0.52 & 0.22& 858 & 275 & 4014 & 1743 & 0.883 \\
\hline
5 & 1.17  & 0.66 & 1386 & 489 & 8472 & 3982 & 0.804 \\
\hline
6 & 2.42 & 1.61  & 2046 & 738 & 15093 & 7166 & 0.833 \\
\hline
7 & 5.89 & 3.52 & 2838 & 1105 & 24651 & 12898 & 0.842 \\
\hline
8 & 9.76 & 6.94 & 3762 & 1494 & 37530 & 20111 & 0.882 \\
\hline
9 & 17.41 & 10.22& 4818 & 1935 & 55488 & 29756 & 0.875 \\
\hline
10 & 28.59 & 17.39  & 6006 & 2436 & 76911 & 41707 & 0.811 \\
\hline
11 & 44.78 & 25.33 & 7326 & 2993 & 103512 & 56782 & 0.875 \\
\hline
12 & 63.37 & 41.93 & 8778 & 3606 & 135201 & 74657 & 0.811 \\
\hline
13 & 92.55 & 63.59  & 10362 & 4275 & 173256 & 96440 & 0.875 \\
\hline
14 & 139.76 & 89.31 & 12078 & 5000 & 217203 & 121415 & 0.811 \\
\hline
15 & 183.25 & 127.57 & 13926 & 5781 & 268752 & 151130 & 0.875 \\
\hline
16 & 246.73 & 193.29 & 15906 & 6618  & 326949 & 184381  & 0.811  \\
\hline
\end{tabular}
% \vspace{-20pt} 
\end{table}

To test the scalability of our optimal robust strategy synthesis approach, we evaluate networks having from 4 to 16 warehouses, with the results reported in Table~\ref{tab:experiments}. 
The different columns report the running time in seconds of the synthesis pipelines, the number of states and transitions in the unrolled MDPST, and the computed probability value.
Moreover, ``na\"ive'' refers to the baseline pipeline consisting of unrolling and Bellman updates (Section~\ref{sec:unroll}), 
while ``pruned'' corresponds to the pruning-based pipeline composed by pruning (Section~\ref{sec:optimizedStrategySynthesis}), unrolling, and Bellman updates. 

The results in Table~\ref{tab:experiments} clearly demonstrate that the pruning-based method presented in Section~\ref{sec:optimizedStrategySynthesis} consistently reduces the state space and achieves notable runtime improvements over the na\"ive approach, without compromising the correctness of the synthesized strategy.

\section{Conclusion}
\label{sec:conclusion}

In this paper, we considered the robust planning for robotic systems with mixed uncertainty and constrained resources.
To this end, we formulated a novel resource-constrained optimal robust strategy synthesis problem over CMDPSTs and $\ltlf$ specifications.
CMDPSTs support modelling adversarial transition uncertainty, state-dependent resource consumption, reload states, and bounded capacity, enabling a formal treatment of the considered systems in a unified framework. 
Then we proposed a synthesis framework based on an unrolled product construction between the CMDPST and a DFA encoding the $\ltlf$ formula. 
To address the challenge of adversarial transitions and unknown successor sets in MDPSTs, we further developed a feasible region-based pruning to reduce the size of the analyzed system while preserving its feasibility properties.
Experiments confirmed the practicality of our approach.

\paragraph*{Aknowledgement}
This work was supported in part by  the CAS Project for Young Scientists in Basic Research (Grant No.\@ YSBR-040), and the National Key R\&D Program of China (Grant No.\@ 2025YFE0220300).
% \newline\protect\includegraphics[height=8pt]{jpg/EU.jpg} 
% This project is part of the European Union’s Horizon 2020 research and innovation programme under the Marie Sk\l{}odowska-Curie grant no.\@ 101008233.

%\input{LP algorithm}

\bibliographystyle{IEEEtran}
\bibliography{bibfiles/references-slim}

\end{document}